\pdfoutput=1
\RequirePackage{fix-cm}
\documentclass[twocolumn]{svjour3}
\smartqed

\usepackage{graphicx}

\usepackage{url}
\usepackage{xcolor}
\usepackage{hyperref}
\hypersetup{
     colorlinks=true,
     linkcolor=orange,
     filecolor=orange,
     citecolor=orange,      
     urlcolor=orange,
     }
\usepackage{color, colortbl}
\usepackage{amsfonts}
\usepackage{amsmath}
\usepackage{appendix}
\usepackage{algorithm, setspace}
\usepackage{algpseudocode}
\usepackage{csquotes}
\usepackage{amsfonts}
\usepackage{booktabs}
\usepackage{siunitx}
\usepackage{breqn}
\usepackage{multirow}
\usepackage[normalem]{ulem}
\usepackage{tikz}
\usetikzlibrary{tikzmark}

\newcommand{\p}[1]{\medskip \noindent \textbf{{#1}.}}
\newcommand{\eq}[1]{Equation~(\ref{eq:#1})}
\newcommand{\fig}[1]{Figure~\ref{fig:#1}}

\journalname{Autonomous Robots}

\DeclareUnicodeCharacter{2212}{-}

\begin{document}

\title{Towards Balanced Behavior Cloning from Imbalanced Datasets}

\author{Sagar Parekh, Heramb Nemlekar, and Dylan P. Losey}

\institute{
    S. Parekh \at
    Mechanical Engineering Department, Virginia Tech \\
    \email{sagarp@vt.edu}
    \and
    H. Nemlekar \at 
    Mechanical Engineering Department, California State University, Northridge. \email{heramb.nemlekar@csun.edu}
    \and 
    D. Losey \at
    Mechanical Engineering Department, Virginia Tech \\
    \email{losey@vt.edu}
}

\maketitle

\begin{abstract}

Robots should be able to learn complex behaviors from human demonstrations. In practice, these human-provided datasets are inevitably \textit{imbalanced}: i.e., the human demonstrates some subtasks more frequently than others. State-of-the-art methods default to treating each element of the human's dataset as equally important. So if --- for instance --- the majority of the human's data focuses on reaching a goal, and only a few state-action pairs move to avoid an obstacle, the learning algorithm will place greater emphasis on goal reaching. More generally, misalignment between the relative amounts of data and the importance of that data causes fundamental problems for imitation learning approaches. In this paper we analyze and develop learning methods that automatically account for mixed datasets. We formally prove that imbalanced data leads to imbalanced policies when each state-action pair is weighted equally; these policies emulate the most represented behaviors, and not the human's complex, multi-task demonstrations. We next explore algorithms that rebalance offline datasets (i.e., reweight the importance of different state-action pairs) without human oversight. Reweighting the dataset can enhance the overall policy performance. However, there is no free lunch: each method for autonomously rebalancing brings its own pros and cons. We formulate these advantages and disadvantages, helping other researchers identify when each type of approach is most appropriate. We conclude by introducing a novel meta-gradient rebalancing algorithm that addresses the primary limitations behind existing approaches. Our experiments show that dataset rebalancing leads to better downstream learning, improving the performance of general imitation learning algorithms without requiring additional data collection. See our project website: \url{https://collab.me.vt.edu/data_curation/}.
\end{abstract}


\maketitle

\begin{figure*}
    \centering
    \includegraphics[width=0.65\linewidth]{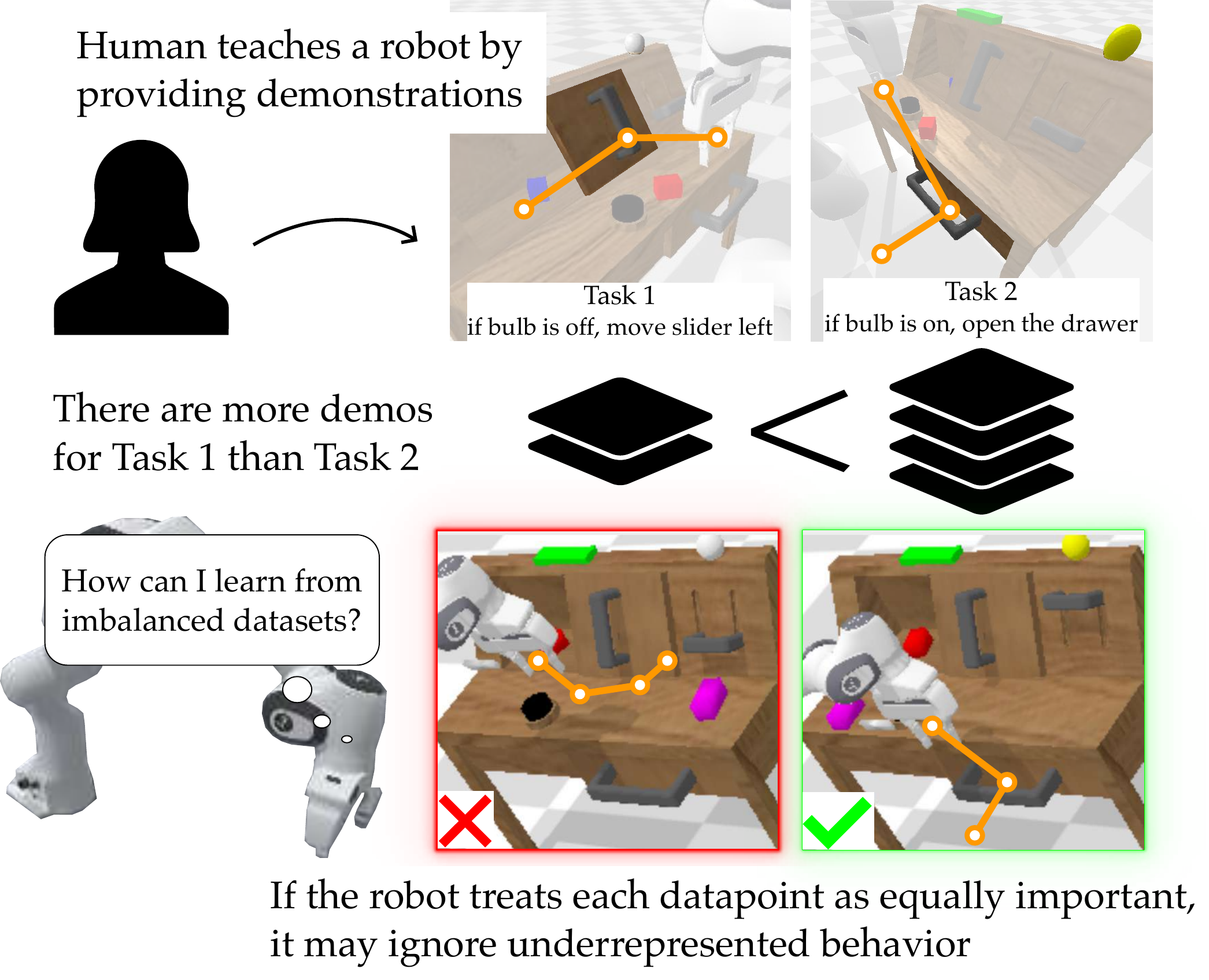}
    \caption{Robot learning how to open a drawer and move a slider from offline human demonstrations. Standard imitation learning equally weights each state-action pair. This results in a policy that imitates the dominant movement patterns seen in the dataset. When the dataset represents one behavior more commonly, imitation learning will learn that behavior at the cost of the other. However, the underrepresented behaviors may not necessarily be unimportant. For instance, opening a drawer and moving the slider require the robot to execute very different motions in different parts of the state space. When the dataset has disproportionately more demonstrations for opening drawer, by imitating the dominant behavior the robot learns to open the drawer but can struggle with moving the slider. So, how do we learn a balanced policy from imbalanced datasets?}
    \label{fig:front}
\end{figure*}

\section{Introduction}

Human teachers have multiple tasks that they want their robots to learn.
Humans therefore provide a dataset of examples, showing the different tasks and their desired behavior across these tasks.
Within the human's examples there are some movement patterns that show up frequently, and other patterns that are only demonstrated once or twice.
For instance, consider a human teaching a robot arm to organize a desk (see \fig{front}).
This may require the robot to pick up objects, open drawers or cabinets, and then put away those objects.
On the one hand, there are certain behaviors --- like reaching for the desired item --- that are commonly observed across all the human's demonstrations.
On the other hand, there are alternate behaviors that are seen rarely in the demonstrations, but are nevertheless important for the robot to learn.
For instance, after grasping the desired object, the three tasks --- opening a drawer, moving a slider, picking up an object --- all require the robot to move in different directions.

More generally, when learning complex or multi-part tasks, robot learners are often provided with \textit{imbalanced} datasets.
Some subtasks are overrepresented in these datasets, while other subtasks are underrepresented.
For instance, driving datasets predominantly contain everyday on-road conditions with very few rare yet safety-critical scenarios, such as occluded and jaywalking pedestrians~\cite{samadi2024good}.

As we will show --- when using standard imitation learning algorithms --- robots trained on imbalanced datasets learn to focus on the commonly represented movement patterns, often ignoring the infrequent behaviors.
Returning to our desk example: if the robot sees a disproportionately large number of demos for opening the drawer, but very few for moving the slider, the robot may overfit to the drawer task while failing to move the slider (see \fig{front}).
This is a practical issue even when the robot is learning a single task.
Take the case of picking up objects: if majority of the human's demonstrations are for picking up objects that are placed on the left side of the table, the common behavior would be for the robot arm to move to towards the left.
A robot trained on this dataset could grasp the common behavior (e.g., left objects), but struggle when it encounters objects on the right.
Other instances of data imbalance include learning tasks with constraints (e.g., reaching a goal while avoiding obstacle regions), or when learning from multiple human teachers with different skill levels (e.g., some users provide optimal demonstrations while others are noisy).

This issue has been known to practitioners and is particularly relevant to recent efforts in collecting large-scale robot datasets containing multiple tasks~\cite{o2024open,khazatsky2024droid}. 
However, there is little theoretical understanding of how data imbalance affects the robot's policy and how it can be effectively mitigated ---
practitioners often balance their data manually based on intuition to ensure equal learning of all behaviors~\cite{team2024octo}.

In this paper, we theoretically investigate the fundamental problem of imitation learning from imbalanced datasets. We begin by showing that the problem arises from a key assumption in standard imitation learning algorithms that each data point contributes equally to the learning process. Specifically, we theorize that:
\begin{center}
    \textit{Equally weighing each state-action pair in an imbalanced dataset leads to imbalanced learning.}
\end{center}
We then explore different strategies for weighing the data samples to learn a more balanced policy. 
Our analysis indicates that the optimal balance depends on two factors: the relative difficulty of learning the behaviors in the dataset and how accurately we want to learn each behavior.
Based on these findings, we highlight the limitations of current methods for automatically balancing robot datasets and introduce a new meta-gradient approach to address them.

Overall, we make the following contributions:


\p{Formalizing Behaviors} We formalize the dataset of demonstrations as a collection of sub-policies. Each sub-policy represents distinct behaviors required to complete the overall task.
For a robot to perform effectively, it must learn each sub-policy regardless of any imbalance in the dataset.

\p{Theoretical Analysis of Imbalanced Learning} We theoretically demonstrate that training on an imbalanced dataset leads to policies that are biased towards more frequently occurring behaviors.
We validate these findings through experiments, empirically showing that imbalance can impair the robot's ability to generalize across all behaviors in the dataset.

\p{Important Considerations for Balancing Datasets} We provide a theoretical examination of existing approaches that balance datasets and demonstrate experiments highlighting their limitations.
Further, we provide insights into the important considerations for balancing datasets.

\p{Meta-Gradient Method for Balancing} To overcome the limitations of prior methods, we introduce a meta-gradient approach for balancing datasets.
This method iteratively calculates the best possible performance achievable for each behaviors and leverages that to balance the dataset, enabling the policy to learn rare behaviors without overfitting to frequent ones.

\section{Related Work}

Imitation learning enables robots to acquire new skills by mimicking the demonstrations of a human expert and potentially generalizing beyond what was shown \cite{osa2018algorithmic}. 
However, its effectiveness is heavily dependent on the quality of the data that is provided to the robot \cite{samadi2024good,hejna2025robot}. 
In this section, we first review prior work that examines how different data attributes impact the learned robot behavior, and then focus on approaches that address the challenge of learning from imbalanced datasets containing a mixture of behaviors.

\subsection{Assessing Data Quality in Imitation Learning}

Previous research has proposed various metrics to quantify what makes a good dataset for imitation learning.
One well-known factor is data optimality, i.e., whether the demonstrations accurately depict the desired behavior \cite{gavenski2022resilient,sagheb2025counterfactual,zare2024survey}.
Another closely related property is consistency in the demonstrated actions, which reduces the ambiguity in how a behavior must be performed \cite{gandhi2023eliciting}.
Both these attributes determine how well the robot can imitate the demonstrated behavior.

However, simply replicating what is shown in the demonstrations is not enough --- one of the main goals in imitation learning is to transfer the learned behavior to new situations. 
To make this possible, the dataset needs to cover a wide variety of states or scenarios \cite{belkhale2024data,shi2025diversity}. 
This need for diversity has led to a push for creating large-scale robot datasets \cite{o2024open,khazatsky2024droid} that include demonstrations from many different tasks and environments. 
Yet, when these datasets are directly utilized with standard imitation learning algorithms, they do not always lead to improved performance unless the data is hand-selected, for example, when training OpenVLA \cite{kim2024openvla} the DROID dataset \cite{khazatsky2024droid} was significantly downsampled to improve policy performance.

Recent work suggests this may be caused by common, repetitive behaviors in the dataset that dominate the learning process \cite{hejna2024re}, forcing designers to manually reduce their weightage \cite{team2024octo}.
This indicates that diversity itself is not sufficient: another important aspect of the training data is the \textit{relative proportion} of data points representing distinct behaviors.
While factors such as optimality and diversity have received significant attention, the effects of varying proportions of different behaviors within a dataset remain underexplored.
We therefore seek to analyze how these proportions influence the robot's learning process, and explore how robots can learn more effectively even when the data is imbalanced.

\subsection{Learning from Imbalanced Data}

The challenge of learning from imbalanced data has been widely studied outside of the robotics domain, in image classification tasks \cite{fernandez2018learning}. Existing methods address this issue by adjusting the data proportions or their learning costs \cite{johnson2019survey,altalhan2025imbalanced}. 
Data-level approaches include undersampling the majority classes \cite{devi2020review,koziarski2020radial,lin2017clustering} or oversampling the minority class to balance the data \cite{fernandez2018smote,dablain2022deepsmote,douzas2018improving}.
On the algorithm side, cost-sensitive learning methods achieve balance by assigning higher weights or penalties to rare, often misclassified, samples to increase their importance during training \cite{khan2018cost}.
In theory, upsampling the data is equivalent to increasing its weight in the learning objective \cite{obuchi2024resampling}.

While the problem of data imbalance has been extensively studied in classification problems, it has not achieve the same rigorous examination in robotics, specifically, imitation learning.
Imitation learning being a regression problem where we must iteratively predict an action (or a distribution over the action space in the case of stochastic policies) differs fundamentally from classification where we learn a decision boundary to group data into categories.
Consequently, the methods to address data imbalance must be tailored to the robotics domain.
There have been some research that focuses on balancing the data and these existing approaches fall into two categories.
Some assume prior knowledge of the distinct behaviors in the training data \cite{hejna2024re,agia2025cupid}, while others learn to recognize the behaviors using a small, additional labeled dataset \cite{fu2023ess,xu2022discriminator,du2023behavior,jang2024safedice}.
For instance, \cite{xu2022discriminator,zhang2023discriminator} trains a discriminator to distinguish between optimal and sub-optimal demonstrations in a large unlabeled dataset using a small amount of expert data for supervision.
Once the behaviors are identified, the dataset is balanced by either collecting additional data to fill the gaps \cite{hou2023shaping,zha2025guiding}, or by re-weighting the existing data \cite{hejna2024re,chen2025curating}.
In \cite{xu2022discriminator}, the discriminator output is used to weigh the behavior cloning loss for each data sample.

A central challenge across these methods is deciding the \textit{desired balance} --- e.g., selecting appropriate weights for each behavior or determining how frequently each behavior should be sampled. 
Furthermore, most prior work focuses on data imbalance in terms of data optimality --- the dataset contains a few optimal demonstrations and a large number of suboptimal demonstrations.
Recently, though, new research has emerged that explicitly addresses data imbalance with respect to behaviors.
We next review these strategies.

\subsection{Determining the Desired Balance}

Deciding how the different behaviors in a dataset should be weighed requires quantifying their `goodness' with respect to some target domain. 
A straightforward approach is to execute the current policy in the real world and change weights based on observed performance \cite{argall2009automatic,agia2025cupid,chen2025curating}.
Alternatively, we can determine weights based on some reference data \cite{xu2022discriminator} or by assuming access to an advantage function for rating the demonstrated actions \cite{wang2018exponentially}.
While effective, these approaches can be costly and impractical due to their reliance on real-world evaluations or specialized domain knowledge.

This raises the question of whether biases in the initial training dataset can be mitigated without relying on additional task-specific feedback.
Prior works have suggested reweighting the data based on training performance \cite{hejna2024re,dass2025datamil}.
Specifically, \cite{dass2025datamil} uses held-out demonstrations to track training progress and increases the learning weights of samples that perform poorly. 
On the other hand, \cite{hejna2024re} computes training loss with respect to a reference policy that is trained on the original, unweighted dataset.
Each of these approaches has limitations. For example, the latter approach fails to learn a good reference when a small proportion of the dataset is sub-optimal, while the former is susceptible to distribution shift issues typical of imitation learning. 

Hence, in this work, we characterize various forms of balance that can be achieved through monitoring training loss. 
Building on this, we propose a novel meta-gradient approach to learn each behavior equally well with respect to its best possible training performance.
Overall, our work provides insights into the theoretical and practical benefits of different approaches for balancing robot data in the absence of target information.
\section{Problem Statement}

We consider settings where a human teaches a robot.
The human provides demonstrations, and the robot learns a policy from these demonstrations.
Each demonstration is a sequence of state-action pairs, i.e., at a given state, the human shows the robot which action to take. 
Typically, these state-action pairs do not follow a single pattern or behavior as the human can have different reasons for choosing actions in different states.
Their decisions are influenced by various factors such as the task they are demonstrating (in multi-task demonstrations), by the constraints in the environment (e.g., an obstacle they must avoid), their skill level (how optimally can they demonstrate), etc.
As an example, when demonstrating how to pick up an object from the table, the human's behavior is dictated by the location of the object.
Consequently, the human would take different actions depending on whether the object is on the left side or the right side of the table.
To imitate the intended task, the robot must capture all the different behaviors present in the demonstrations.
In this section, we formalize these different behaviors as sub-policies of the human, and discuss how existing methods learn a policy from the human's data. 

\p{Robot} Let $s \in \mathcal{S}$ be the system state and let $a \in \mathcal{A}$ be the the robot's action.
The state updates according to the dynamics: $s' = T(s, a)$.
We can observe the system state $s$, but we do not have access to the dynamics $T$.

\p{Human} The human teacher demonstrates their desired behaviors to the robot learner.
Specifically, the human provides an offline dataset of expert state-action pairs: $\mathcal{D} = \{(s_1, a_1) \ldots (s_N, a_N)\}$. 

When providing these demonstrations, the human has some overall policy $\pi_h(a \mid s)$ in mind. This policy models the primary task they want the robot to perform (e.g., picking up an object) as well as additional conditions that the robot should satisfy (e.g., do not run into the table).
Alternatively, the human may want to teach multiple tasks to the robot, in which case the policy may choose different actions in a state depending on the task (e.g., open a drawer and move a slider).
Without loss of generality, we can write $\pi_h(a \mid s)$ as a combination of $k$ different components or sub-policies $\pi_{i}$ that represent the distinct behaviors that humans exhibit in different subsets of the state space $\mathcal{S}_i \subseteq \mathcal{S}$.
\begin{equation*}
    \pi_h(a \mid s) = \pi_i(a \mid s_i), \text{ where } s_i \in \mathcal{S}_i
\end{equation*}
We assume that the state spaces for each sub-policy are disjoint, i.e., $\bigcap\limits_{i=1}^{k} \mathcal{S}_i = \emptyset$. For example, the human would choose to move to the left is states $\mathcal{S}_1$ where the object is placed on the left, and then take actions to move right in states $\mathcal{S}_2$ where the object is placed on the right.
In case the human is suboptimal and chooses to move right even when the object is on the left, we can consider them as a third subset $\mathcal{S}_3$. 
Conversely, if the data does not contain multiple subtasks or there are no special conditions to be met (e.g., thee object is always in the same location), the human policy $\pi_h(a \mid s)$ will reduce to a single sub-policy over the states $\mathcal{S}_i = \mathcal{S}$.

\p{Policy} The robot learns a control policy $\pi_{\theta}(a \mid s)$ from dataset $\mathcal{D}$.
This policy is instantiated as a neural network with weights $\theta$.
Ideally, this learned policy should match the expert's policy $\pi_h(a \mid s)$.
Since we formulate $\pi_h(a \mid s)$ as a combination of $k$ sub-policies, the learned policy should match all $k$ sub-policies of the expert to effectively imitate each demonstrated task.
Typically, imitation learning is formulated as a supervised learning problem where we maximize the likelihood of action $a$ observed in the dataset given state $s$
\begin{equation} \label{eq:P1}
    \mathcal{L}(\theta) = -\mathop{\mathbb{E}}\limits_{(s, a) \in \mathcal{D}} \log{\pi_\theta(a \mid s)}
\end{equation}
Optimizing this objective is equivalent to minimizing the KL divergence between the human policy and the robot policy. This is shown below:
\begin{align*}
    &D_{KL} (\pi_h(a \mid s) \mid\mid \pi_\theta(a \mid s)) = \mathop{\mathbb{E}}\limits_{(s, a) \in \mathcal{D}} \log \frac{\pi_h(a \mid s)}{\pi_{\theta}(a \mid s)} \\
    &= \mathop{\mathbb{E}}\limits_{(s, a) \in \mathcal{D}} \log \pi_h (a \mid s) - \log \pi_\theta (a \mid s) \\
    &= -\mathop{\mathbb{E}}\limits_{(s, a) \in \mathcal{D}} \log{\pi_\theta(a \mid s)}
\end{align*}
We can ignore the $\log \pi_h(a \mid s)$ term since it does not depend on $\theta$.
Hence, to learn the weights $\theta$, we minimize the KL divergence between the human and the robot policy, i.e., the objective for imitation learning is:
\begin{equation} \label{eq:P2}
    \mathcal{L}_{BC} = \mathop{\mathbb{E}}\limits_{(s, a) \in \mathcal{D}} D_{KL} (\pi_h(a \mid s) \mid\mid \pi_\theta(a \mid s))
\end{equation}

This expectation is calculated by equally weighing each state-action pair in the dataset.
Minimizing this objective allows the robot policy to match the human's policy $\pi_h(a \mid s)$ and mimic the human's actions across all states in $\mathcal{S}$.
But how well does the policy learn to mimic each of the human's individual sub-policies $\pi_i(a \mid s)$?
Remember that the dataset can be imbalanced with some states $s_i$ occurring more frequently than others.
In the following section, we theoretically examine whether a robot policy trained using standard behavior cloning objective can correctly imitate all behaviors in the dataset.

\section{Learning from Heterogeneous Data}\label{sec:analysis}

We examine the problem of learning a policy from an offline dataset $\mathcal{D}$ of state-action pairs.
As we defined in the previous section, this dataset is obtained from human teachers who act based on a combination of $k$ behaviors or sub-policies.
Particularly, each sub-policy $\pi_i(a | s)$ governs the human's behavior in a distinct region of the state space $\mathcal{S}_{i}$. So each state-action pair in $\mathcal{D}$ is associated with a specific sub-policy. 
In general, the dataset may contain different amounts of state-action samples for each sub-policy.
In this section, we use $\rho_{i}$ to denote the proportion of dataset samples that belong to a sub-policy $\pi_{i}$, and analyze how these proportions impact the robot's learning. 
Specifically, we first determine how the parameters of the robot's policy are biased by an unequal distribution of behaviors when using standard behavior cloning. We then conduct experiments to show how data distribution impacts performance in a simulated robot manipulation task.

\subsection{How does Data Imbalance Affect Policy Parameters?}

We begin by deriving the policy learned by the robot when applying standard behavior cloning approaches to the dataset $\mathcal{D}$.
Consider the behavior cloning objective from \eq{P2}. 
To understand how this objective is related to the behavior proportions $\rho_{i}$, we will unpack the expectation term as follows.
The expectation is evaluated over all state-action pairs $(s, a)$, i.e., it weighs each sample with the joint probability $p(s, a)$ of observing it in the dataset. 
We can rewrite this equation by expanding the expectation and decomposing the human policy into a mixture of multiple sub-policies as:
\begin{align} \label{eq:P3}
    \mathcal{L}_{BC} &= \sum\limits_{(s, a) \in \mathcal{D}} p(s, a) D_{KL} (\pi_h (a | s) \mid\mid \pi_\theta (a | s)) \nonumber \\
    &= \sum\limits_{i=1}^{k} \sum_{(s, a) \in \mathcal{D}_{i}} p(s, a) D_{KL} (\pi_i \mid\mid \pi_\theta)
\end{align}
The second step follows from the observation that the human employs each sub-policy in a distinct region of the state space, and so we divide the dataset into subsets $\mathcal{D}_{i}$ containing states and actions corresponding to sub-policies $\pi_{i}$. Here we use $\pi_{i}$ and $\pi_{\theta}$ as simplified notations for $\pi_{i}(a|s)$ and $\pi_{\theta}(a|s)$.

We can further break this down by applying Bayes' rule to express the joint probability as a product of the probability of a state-action pair belonging to a sub-policy subset $p(\mathcal{D}_{i})$, and the conditional probability of observing that sample given that subset, $p(s, a|\mathcal{D}_{i})$.
\begin{align} \label{eq:P4}
    \mathcal{L}_{BC} &= \sum\limits_{i=1}^{k} \sum_{(s, a) \in \mathcal{D}_{i}} p(s, a | \mathcal{D}_{i})p(\mathcal{D}_{i}) D_{KL} (\pi_i \mid\mid \pi_\theta) \nonumber \\
    &= \sum\limits_{i=1}^{k} \rho_{i} \mathop{\mathbb{E}}_{(s, a) \in \mathcal{D}_{i}} D_{KL} (\pi_i \mid\mid \pi_\theta)
\end{align}
Here we recognize that $p(\mathcal{D}_{i})$ represents the proportion $\rho_{i}$ of data samples that belong to the sub-policy $\pi_{i}$. The expectation in \eq{P4} is now over the subset $\mathcal{D}_{i}$ rather than over the entire data as in \eq{P2}.
From this objective, we see that standard behavior cloning weighs the divergence between the robot's policy $\pi_{\theta}$ and the sub-policy $\pi_{i}$ by the frequency of observing that sub-policy. 
But how does this weighted objective affect the learned policy?

We theorize that a policy trained on this objective is more likely to mimic the prominent behaviors in the dataset while ignoring less frequent ones. To formalize this intuition, we analytically calculate the optimal parameters of the robot policy for standard behavior cloning. 
Specifically, we adopt a univariate Gaussian model for both human and robot policies and present the following result:

\begin{proposition} \label{prop:1}
    Let the robot's policy $\pi_{\theta}$ and the human's $k$ sub-policies be Gaussian with parameters $(\theta, \sigma)$ and $(\theta_i, \sigma_i)$:
    \begin{align*}
        \pi_{\theta}(a | s) &= \mathcal{N}(\theta s, \sigma), \text{ where } s \in \mathcal{S} \\
        \pi_i(a | s) &= \mathcal{N}(\theta_i s_i, \sigma_i), \text{ where } s_i \in \mathcal{S}_i \subset \mathcal{S}
    \end{align*}
    Then, using standard behavior cloning, the learned parameters $\theta$ of the robot's policy are a weighted sum of the sub-policy parameters $\theta_{i}$, where each weight $\rho_{i}$ is the joint probability of states and actions from the sub-policy $\pi_i$.
    \begin{align*}
        \theta = \sum\limits_{i=1}^k \rho_i \cdot \theta_i
    \end{align*}
    Hence, the behavior cloning objective biases the robot's policy towards more frequently observed sub-policies.
\end{proposition}

\begin{proof}
    We start with the behavior cloning loss in \eq{P4} and use the analytical form of KL-divergence for Gaussian distributions to express it as a function of the policy parameters.
    \begin{align} \label{eq:P5}
        \mathcal{L}_{BC} &= \sum\limits_{i=1}^{k} \rho_{i} \mathop{\mathbb{E}}_{(s, a) \in \mathcal{D}_{i}} \Bigg(\log\frac{\sigma}{\sigma_i} + \frac{\sigma_{i}^{2} + (\theta_{i} s - \theta s)^{2}}{2\sigma^{2}} - \frac{1}{2} \Bigg) \nonumber\\
        &\propto  \sum\limits_{i=1}^{k} \rho_{i} \mathop{\mathbb{E}}_{(s, a) \in \mathcal{D}_{i}} (\theta_{i}s - \theta s)^{2}
    \end{align}
    We obtain the second step by dropping all terms that are not dependent on the learnable parameters $\theta$.
    Since KL-divergence is convex, we can now derive the optimal parameters by taking the gradient of the simplified objective in \eq{P5} with respect to $\theta$ and equating it to zero.
    \begin{align*}
        \nabla_{\theta} \mathcal{L}_{BC} &= 0  \\
        \sum\limits_{i=1}^{k} \rho_{i} \mathop{\mathbb{E}}_{(s, a) \in \mathcal{D}_{i}} 2(\theta - \theta_{i}) s^{2} &= 0 \\
        \theta \sum\limits_{i=1}^{k} \rho_{i} \left( \mathop{\mathbb{E}}_{(s, a) \in \mathcal{D}_{i}} s^{2} \right) - \sum\limits_{i=1}^{k}  \rho_{i}\theta_{i} \left(\mathop{\mathbb{E}}_{(s, a) \in \mathcal{D}_{i}} s^{2} \right) &= 0
    \end{align*}
    In the last step, we take $\theta$ outside the expectation since it is independent of the states and actions. The $s^{2}$ term represents the magnitude of the task states. In practice, we can normalize the state values such that $\mathop{\mathbb{E}}_{(s, a) \in \mathcal{D}_{i}} s^{2} = 1$. With this assumption and knowing that the relative proportions sum to one (i.e., $\textstyle\sum_{i=1}^{k} \rho_{i} = 1$), we can rearrange the terms to calculate the optimal parameters as:
    \begin{equation} \label{eq:P6}
        \theta = \sum_{i=1}^k \rho_i \cdot \theta_i
    \end{equation}
    Therefore, we find that the learned policy parameters are a weighted sum of the parameters of each individual sub-policy, where the weights $\rho_{i}$ are the probability density of the states and actions associated with that sub-policy.
\end{proof}

Proposition \ref{prop:1} demonstrates how the relative proportions of behaviors in the dataset can influence the robot's policy when using standard behavior cloning. If the data contains more state-action pairs for a particular sub-policy than others, that behavior will have a higher weight $\rho_{i}$ and will thus play a greater role in deciding the robot's actions.
On one hand, this is reasonable because we want the robot to be better at the behaviors that it encounters more often. But on the other hand, the frequency of a sub-policy may not necessarily indicate its importance; e.g., the underrepresented sub-policies may encode safety-critical behaviors that are equally important to the task performance. 

So far, our analysis has assumed Gaussian policies with linear parameters. However, our insights can be extended to non-linear policies modeled using neural networks.
We do this by evaluating the worst-case expected loss for each sub-policy in \eq{P4}. 
Let $\mathcal{L}_{BC} = L$ be the loss at which the training converges.
The worst case for a sub-policy $\pi_{i}$ will be when this loss is entirely because of that sub-policy, while all other sub-policies are learned perfectly such that $D_{KL}(\pi_{j\neq i}||\pi_{\theta}) \rightarrow 0$ and:
\begin{equation*}
L = \rho_{i} \cdot \mathbb{E}_{(s,a)\in\mathcal{D}_{i}} D_{KL}(\pi_{i}||\pi_{\theta})    
\end{equation*}
More generally, the training loss for each sub-policy will be bounded as:
\begin{equation} \label{eq:P8}
\mathbb{E}_{(s,a)\in\mathcal{D}_{i}} D_{KL}(\pi_{i}||\pi_{\theta}) \leq \frac{\mathcal{L}_{BC}}{\rho_{i}}
\end{equation}
This upper bound increases as we decrease the proportion $\rho_{i}$ of samples, making it difficult in practice to accurately learn the underrepresented behaviors. 

Overall, \eq{P6} and \eq{P8} define how the dataset proportions bias the policy learned by the robot when using standard behavior cloning.
In what follows, we will experimentally demonstrate how this can negatively impact the policy performance.

\subsection{Experiments with Imbalanced Data}\label{sec:sim_1}
Our analysis established how the policy learned with standard behavior cloning is affected by the relative proportions of behaviors in the training dataset.
When these proportions are imbalanced, we expect the policy to under-perform on the underrepresented behaviors, leading to asymmetric rollouts.
Here we empirically demonstrate this effect through controlled simulation experiments.

\p{Environment and Task} We conduct our experiments in the CALVIN environment \cite{mees2022calvin}, an open-source platform that simulates visually rich, tabletop manipulation tasks.
CALVIN is used as a benchmark simulation for learning long-horizon manipulation tasks.
As illustrated in \fig{task}, the environment includes a $7$ degrees-of-freedom (DoF) Franka Emika Panda robot arm and a variety of interactive objects, including sliders, workbench, differently-colored blocks, buttons and switches for operating lights, and drawers.

\begin{figure*}
    \centering
    \includegraphics[width=1\linewidth]{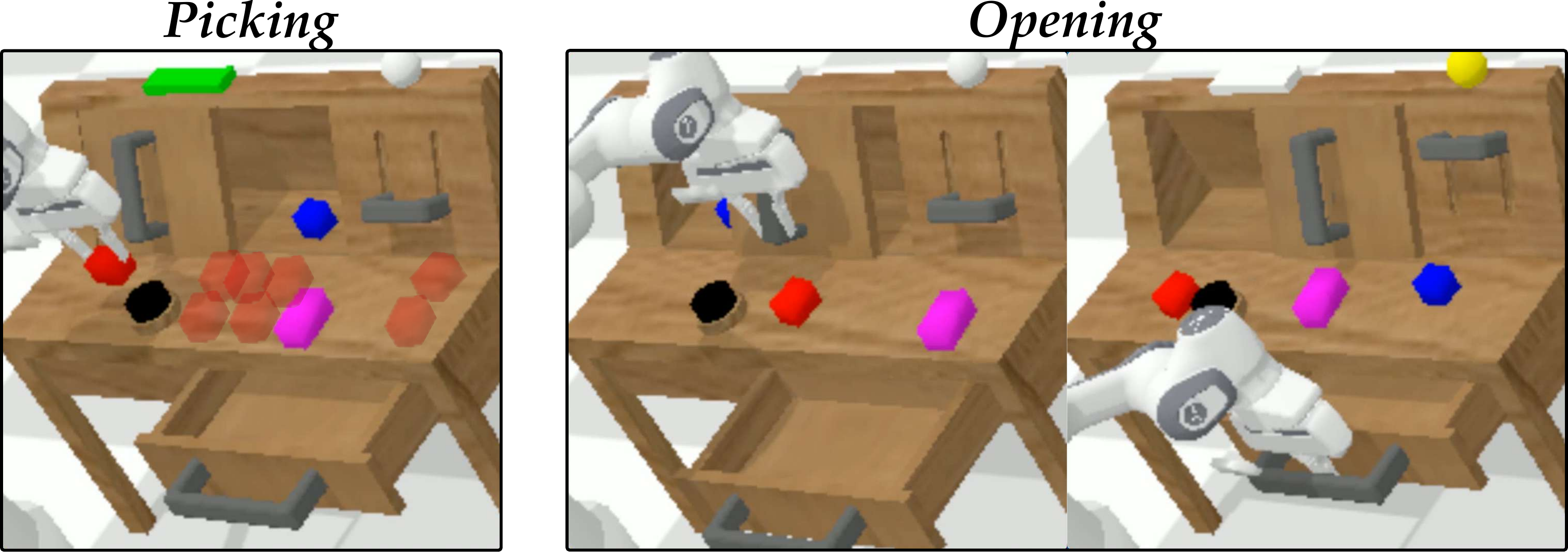}
    \caption{The manipulation tasks we perform in our experiments. On the left, we see \textit{picking} where the robot must learn to grasp and pick up a red block to a certain height. The red block is initialized at a position that is sampled randomly from three regions: left side of the table, middle of the table, and right side of the table. The possible initialization locations of the red block are shown by transparent overlays. On the right is \textit{opening} which consists of two tasks conditioned on the state of the environment. If the bulb is off, the robot must move a slider to the left and if the bulb is on, the robot must open the drawer. All the objects irrelevant to the task are initialized randomly.}
    \label{fig:task}
\end{figure*}

We design two experiments in this environment. First, a \textit{picking} experiment in which the robot is required to reach for a red block placed on the workbench, grasp it, and lift it to a predefined height.
Next, we design a multi-task experiment \textit{opening} where the robot either opens a drawer or moves a slider on the table.
We generate demonstration data by executing a pre-trained expert policy \cite{reuss2024multimodal} in the environment.
At the beginning of each demonstration, the objects --- including the red block --- are initialized in random configuration.
Specifically, in the first task, the red block's position is sampled from three distinct spatial regions: the left, the middle, and the right areas of the robot's workspace, refer to \fig{task} (left).
In the second task, the drawer and slider locations are fixed, however to condition the policy on the task data, we utilize the state of the light bulb: if the light bulb is on, the robot opens a drawer. If the light bulb is off, the robot instead moves a sliding door, see \fig{task} (right).

\p{Sub-policies}
In the \textit{picking} experiment, although the expert policy is trained to solve the task regardless of the block's initial position, its behavior can be decomposed into three distinct motion strategies, each corresponding to one of the spatial regions: moving to the left ($\pi_1^p$), moving to the center ($\pi_2^p$), and moving to the right ($\pi_3^p$).
These three sub-policies reflect meaningful variations in the robot's motion and form the core behaviors that we seek to evaluate.

In the \textit{opening} experiment, the behavior of the expert policy can be decomposed into two distinct behaviors corresponding to moving the slider door ($\pi_1^o)$ and opening the drawer ($\pi_2^o$). 
While the policy needs to learn both these behaviors, either of the two behaviors are distinct from one another on account of the region of the space the robot must visit as well as the type of motion it must execute to complete the task.
For instance, opening the drawer requires the robot to move to the drawer handle from the top and pull it towards it, whereas moving the slider requires the robot to move straight to the slider door and pull it to the side.

\p{Demonstrations}
Each demonstration is recorded as a sequence of observation-state-action tuples: $\xi = \quad$ $\{(o^1, s^1,  a^1), (o^2, s^2, a^2), \cdots, (o^n, s^n, a^n)\}$.
Observations contain a static RGB image $o_{static} \in \mathbb{R}^{200 \times 200 \times 3}$ from a fixed camera overlooking the entire workspace, an egocentric RGB image $o_{ego} \in \mathbb{R}^{84 \times 84 \times 3}$ from a camera mounted on the robot gripper, and a robot state $s \in \mathbb{R}^8$ comprising of the robot's $7$ joint angles and a binary gripper state.
The action $a \in \mathbb{R}^7$ represents the $6$-dimensional linear and angular velocity of the robot's end-effector and a binary gripper actuation command.

\p{Procedure}
To quantify the effect of dataset imbalance, we begin by constructing a \textit{balanced} dataset containing an equal number of demonstrations for each sub-policy ($\pi_1^p, \pi_2^p, \pi_3^p$ in \textit{picking} and $\pi_1^o, \pi_2^o$ in \textit{opening}).
Specifically, in \textit{picking} we collect $21$ demonstrations per sub-policy, resulting in a total of $63$ demonstrations.
In \textit{opening} we collect $15$ demonstrations per sub-policy, resulting in a total of $30$ demonstrations.
We train a standard behavior cloning policy on this balanced dataset which serves as a baseline for comparison.
To introduce imbalance, we selectively reduce the number of demonstrations corresponding to one sub-policy while keeping the total demonstration count fixed.
More particularly, in \textit{picking} we selectively reduce the demonstrations to $9$ for one sub-policy and increase the demonstrations for other sub-policies to $27$.
This ensures that the total number of demonstrations remains the same as $63$.
In \textit{opening} we reduce the demonstrations to $10$ for one sub-policy at a time while increasing it to $20$ demonstrations for the other sub-policy keeping the total at $30$.
This process is applied independently to each of the sub-policies in the dataset, yielding three and two distinct imbalanced datasets for \textit{picking} and \textit{opening}, respectively.
For each imbalanced dataset, we train a separate behavior cloning policy.
These policies are then evaluated on test scenarios that uniformly cover all three behaviors.
To ensure statistical robustness, we repeat the training and evaluation process $10$ times per dataset.
We report the average success rate for each behavior along with the standard deviation.

\begin{figure*}
    \centering
    \includegraphics[width=1\linewidth]{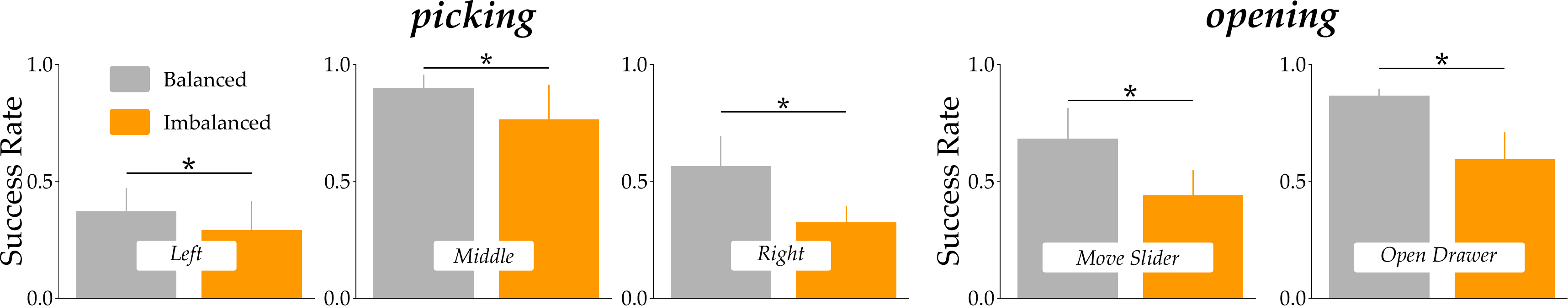}
    \caption{Results of our experiments from Section \ref{sec:sim_1}. We compare the performance of a policy trained on a balanced dataset with one trained on an imbalanced dataset. The balanced dataset contains $21$ demonstrations for each of the three sub-policies in \textit{picking} totaling $63$ demonstrations. In \textit{opening} the balanced dataset contains $15$ demonstrations for each of the two sub-policies totaling $30$ demonstrations. We introduce imbalance by reducing the number of demonstrations for one sub-policy at a time. We train a separate policy for each of the imbalanced datasets as well as a policy that is trained on the balanced dataset. The plots compare the success rate of the policies across $100$ rollouts. The gray bars represent the success rate for the balanced policy and the orange bars represent the success rate of the imbalanced policy for each case. (Left) Results for the three behaviors in \textit{picking}: lifting up the red block that is on the left side, the middle, and the right side of the table. (Right) Results for the two behaviors in \textit{opening}: moving the slider to the left when the bulb is off and opening the drawer when the bulb is on. The vertical bars show the standard deviation and $*$ indicates statistical significance.}
    \label{fig:results_imbalanced}
\end{figure*}

\p{Results}
\fig{results_imbalanced} summarizes the results of our experiments.
The three plots on the left show the success rate for the cases where the initial position of the block is in the region on the left, middle, and right, respectively.
The gray bar in each plot shows the success rate when the policy is trained on the balanced dataset.
The orange bar shows the success rate when the policy is trained on an imbalanced dataset where the specific sub-policy is underrepresented.
The performance of the policy trained on the imbalanced dataset, where the specific behavior is under-represented, is significantly lower than the policy trained on the balanced dataset (left: $t=-2.601, p < 0.05$, middle: $t=-2.344, p < 0.05$, right: $t=-2.893, p < 0.01$). 

The two plots on the right in \fig{results_imbalanced} summarize the results for the \textit{opening} experiment.
The plot on the left shows the success rate of moving the slider when the bulb is off and the plot on the right shows the success rate of opening the drawer when the bulb is on.
We find that the success rate significantly drops when the dataset is imbalanced (moving the slider: $t=-3.728$, $p < 0.01$, opening drawer: $t=-6.657$, $p < 0.001$).

These results support Proposition~\ref{prop:1}, which states that when the dataset consists of multiple behaviors, the proportion of each behavior significantly affects how well the policy can learn each task.
When the behaviors are disproportionately represented in the dataset, the underrepresented behaviors are difficult for the policy to learn reliably.
However, when the dataset is balanced, the policy can achieve better results corresponding to each behavior.

\section{Balancing the Dataset}\label{sec:5}

In the previous section we theoretically and empirically showed that training the robot on a mixed dataset can bias the robot's learning towards behaviors that are more prominent in the dataset. This can lead to poor performance on scarce behaviors that may be equally important.
From \eq{P8}, we observe that we can ensure that a sub-policy $\pi_{i}$ is effectively learned by increasing its weight $\rho_{i}$. However, since the weights represent the relative proportion of data, increasing the weight for one behavior would decrease the weights for the others. 
So how do we \textit{balance} the data to ensure that each behavior is effectively learned?
Balancing in the context of our problem refers to re-weighting the dataset to enhance the policy's overall performance.
In this section we analyze different methods that can be employed to balance a heterogeneous dataset.
Additionally, we discuss the possible applications and limitations of each of these approaches.

\subsection{Equally Weighing Each Behavior}\label{sec:5a}
Following the classification literature \cite{fernandez2018smote,devi2020review}, we can employ the common practice of undersampling the overrepresented behaviors or oversampling the underrepresented behaviors.
This leads to a simple, perhaps naive, approach where we assign equal weights to each sub-policy in the dataset:
\begin{equation} \label{eq:P9}
    \mathcal{L}_{eq-w} = \frac{1}{k} \sum_{i=1}^{k} \mathop{\mathbb{E}}_{(s,a) \in D_{i}} D_{KL}(\pi_{i} \mid\mid \pi)
\end{equation}

In this objective, we effectively change the proportional weighting in \eq{P4} to a uniform weighting, i.e., we replace $\rho_{i}$ with weights $\alpha_{i} = 1/k$ for $i = 1, 2, \cdots, k$. 
To compare this balanced objective with the original behavior cloning objective of \eq{P3}, we reintroduce the sum over state-action pairs and apply Bayes' rule to re-write the objective as 
\begin{equation} \label{eq:P10}
    \mathcal{L}_{eq-w} = \frac{1}{k} \sum_{i=1}^{k} \sum_{(s, a) \in D_{i}} \frac{p(s, a)}{\rho_{i}} D_{KL}(\pi_{i} \mid\mid \pi)
\end{equation}

\eq{P10} implies that, for the robot to learn each sub-policy equally well, the sampling probability for each state-action pair must be scaled down by the frequency of its sub-policy. This is equivalent to sampling the data from a new data distribution $q(s, a) = p(s,a) / \rho_{i}$.
We can prove that this objective enables the robot to learn unbiased parameters by following the same steps as in Proposition 1 to get: 
\begin{equation} \label{eq:P11}
    \theta = \sum_{i=1}^{k} \frac{1}{k} \theta_{i}
\end{equation}
In the non-linear case, this corresponds to having the same worst-case training loss for each sub-policy:
\begin{equation} \label{eq:P12}
\mathbb{E}_{(s,a)\in\mathcal{D}_{i}} D_{KL}(\pi_{i}||\pi_{\theta}) \leq k\mathcal{L}_{eq-w} \quad \forall i \in \{1, \ldots, k\}
\end{equation}

Therefore, \textit{balance} in this case refers to ensuring a common lower-bound on the learning accuracy for each behavior.
As we will demonstrate later in this section, applying $\mathcal{L}_{eq-w}$ for the manipulation task described in Section~\ref{sec:sim_1} results in improved success on both \textit{picking} and \textit{opening} experiments (see \fig{manual_balanced}).

\p{Limitation}
This approach is founded upon two assumptions that limit its scope: ($1$) we want the same bound on training performance for each behavior, and ($2$) it is equally easy to learn each behavior. 

As a counterexample, consider a dataset containing two sub-tasks of varying difficulty: throwing a ball at a moving target (hard) and dropping the ball into a large stationary bin (easy). Typically, the robot would require more data and training iterations to learn the first subtask than the second.
Due to this, when we train the robot using \eq{P9}, despite giving equal weight to both subtasks, the robot may still learn to drop the ball more accurately than throwing, such that $\mathbb{E}[\mathcal{L}_{drop}] \rightarrow 0$ and $\mathbb{E}[\mathcal{L}_{throw}] \rightarrow 2\mathcal{L}_{eq-w}$. 
In this case, we can continue training until $\mathcal{L}_{eq-w} \rightarrow 0$ to get better at throwing the ball, but this may lead the robot to overfit to the dropping subtask.
Another point of consideration is that the same training loss may not guarantee success in both subtasks. For instance, we require greater accuracy to hit the target with the ball than to drop it in a large bin --- thus, we may actually want different training losses for each behavior.

One way we can address both these concerns is by assigning a higher weight to the throwing subtask to learn it more accurately. While this could increase the training loss for the dropping subtask, if done carefully, it may not affect its success rate due to the greater margin of error provided by the size of the static bin. 
We next discuss how to achieve such balance that accounts for both the relative difficulty and precision requirement of the behaviors in the training dataset.

\subsection{Equally Learning Each Behavior (Relatively)}\label{sec:5b}

As discussed in the above example, we may not always want to weigh each behavior equally. In general, we may want some weights $\alpha = [\alpha_{1}, \dots, \alpha_{k}]$ that account for how difficult it is to learn a behavior and how accurately we want to learn it.
We now present an approach for balancing the training data, taking these considerations into account.

From our analysis in Section~\ref{sec:analysis}, we know that if we want to learn a behavior more accurately, we should increase its relative proportion, i.e., its weight.
We can use this insight to dynamically change the weights $\alpha$ during training. 
Specifically, we increase the weights for sub-policies that have a higher expected loss by maximizing the following objective:
\begin{align}\label{eq:P13}
    \mathcal{L}_{eq-l} = \sum_{i=1}^{k} \alpha_{i} \left( \mathbb{E}_{(s,a) \sim\mathcal{D}_{i}} D_{KL} (\pi_{i} || \pi_{\theta}) - \mathcal{L}_{i}^{ref} \right)
\end{align}
Here, $\mathcal{L}_{i}^{ref}$ is the reference loss for each sub-policy and represents its desired accuracy.
The difference:
$$\delta_{i} = \left(\mathbb{E}_{(s,a) \sim\mathcal{D}_{i}} D_{KL} (\pi_{i} || \pi_{\theta}) - \mathcal{L}_{i}^{ref}\right)$$ measures how good or bad the current policy is compared to the target.  
If the current loss is much higher than the reference (i.e., $\delta_{i}$ is high), the learner will try to raise its $\alpha_{i}$ to maximize \eq{P13}.
Choosing a lower reference loss would indicate that we desire greater accuracy for that behavior and vice versa.
In the simplest case, this reference can be set to zero (i.e., $\mathcal{L}^{ref}_{i} = 0$ for $i \in \{1, \ldots, k\}$), indicating that we desire the same training accuracy across all behaviors.

To converge to the optimal weights, we must iteratively update the policy $\pi_{\theta}$ with the changing weights.
This results in a min-max operation where we alternate between two steps: maximizing $\mathcal{L}_{eq-l}$ to update $\alpha$ while keeping $\theta$ constant, and then minimizing $\mathcal{L}_{eq-l}$ with respect to $\theta$ while keeping the updated $\alpha$ fixed.
We can derive the convergence condition for the maximization step by first taking its gradient with respect to $\alpha$:
\begin{equation}\label{eq:P14}
    \nabla_{\alpha} \mathcal{L}_{eq-l} = \left[\delta_{1} + \lambda, \ldots, \delta_{k}  + \lambda \right]
\end{equation}
where $\lambda = -\frac{1}{k}\sum_{i=1}^{k}\delta_{i}$ is a Lagrange multiplier which projects the gradients to keep $\alpha$ within the valid space $\Delta = \{\alpha \mid \sum_{i=1}^{k} \alpha_{i} = 1\}$ of relative proportions.
A more detailed derivation of the gradients is provided in Appendix~\ref{sec:appendix}. 
Equating $\nabla_{\alpha} \mathcal{L}_{eq-l} = 0$ results in the following balanced state at which the weights $\alpha$ converge:
\begin{equation}\label{eq:P15}
    \delta_{1} = \ldots = \delta_{k}
\end{equation}

This means that the robot will try to learn weights that make the difference term $\delta_{i}$ equal for all behaviors.
Therefore, \textit{balance} in this case refers to ensuring that each behavior is learned equally well relative to its corresponding reference $\mathcal{L}_{i}^{ref}$. 
This method of balancing the dataset has been empirically validated in prior work \cite{hejna2024re}, with varying assumptions about the reference loss.
Our contribution supplements these findings by providing a theoretical understanding of the type of balance achieved with this approach.

\p{Limitation}
While $\mathcal{L}_{eq-l}$ addresses the fundamental limitations of the $\mathcal{L}_{eq-w}$, the effectiveness of $\mathcal{L}_{eq-l}$ depends on the choice of $\mathcal{L}_{i}^{ref}$.
An overly \textit{conservative} target makes the difference term smaller, giving less weight to that behavior and causing the robot to learn it less accurately. By contrast, an \textit{overoptimistic} target causes the robot to only focus on that behavior, ignoring other behaviors in the process.
In the following experiments we analyze how this reference loss impacts the robot's learning in the simulated manipulation task.

\subsection{Experiments Comparing Data Balancing}\label{sec:5c}

\begin{figure*}
    \centering
    \includegraphics[width=1\linewidth]{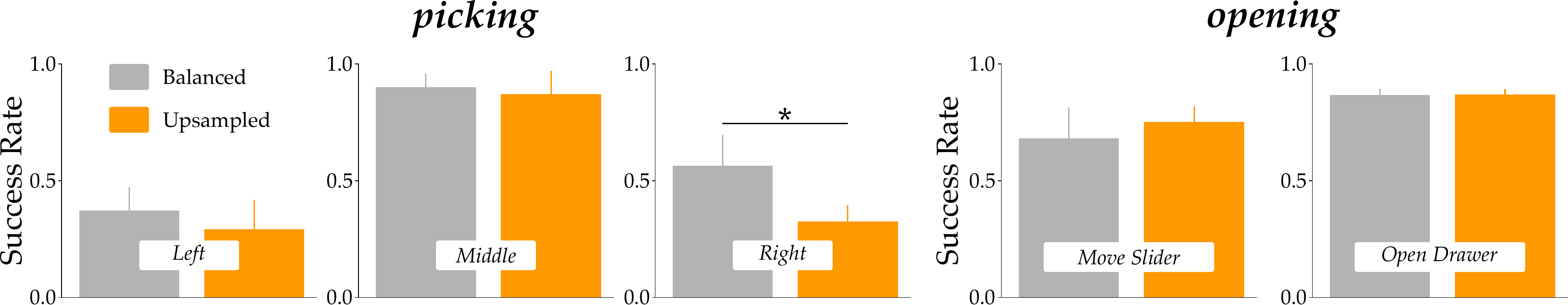}
    \caption{Results of the experiment demonstrating the benefits of upsampling data. We compare the performance of a policy trained on an imbalanced dataset with a policy trained on an upsampled dataset. In the imbalanced dataset one of the sub-policies is underrepresented with fewer demonstrations that the other sub-policies. In the upsampled dataset, we resample the fewer demonstrations to make the proportions of each sub-policy equal such that all sub-policies are equally represented. The plots compare the success rate of the trained policies across $100$ rollouts. Th gray bars represent the success rate for the imbalanced policy and the orange bars represent the success rate for the upsampled policies. (Left) The results for the three sub-policies in \textit{picking}. (Right) Results for the two sub-policies in \textit{opening}. The vertical bars show the standard deviation and $*$ indicates statistical significance.}
    \label{fig:manual_balanced}
\end{figure*}

In this section we empirically examine the effect of balancing the dataset on the policy's performance.
We conduct two experiments, \textit{picking} and \textit{opening} that we detail in Section \ref{sec:5c}, in the CALVIN environment.
To examine the effect of balancing, we perform a series of tests.
First, we examine whether uniformly weighing each sub-policy can achieve the same performance as a policy that is trained on a balanced dataset.
Next, we examine an existing method \textbf{Remix} \cite{hejna2024re} that re-weights the dataset using min-max optimization discussed in Section \ref{sec:5b}.
Specifically, this method uses a reference policy as the target loss $\mathcal{L}^{ref}$.
This reference policy is trained on the original imbalanced dataset.
Finally, we demonstrate a case that limits the applicability of a reference policy as the target loss.

\p{Procedure} In our first set of tests, we want to see if uniformly weighting the data can help the policy learn all the behaviors effectively.
As a baseline, we train a policy on a balanced dataset that has equal proportion of each sub-policy.
For \textit{picking} the dataset contains $21$ demos for each of the three sub-policies and for \textit{opening} the dataset consists of $15$ demos for each of the two sub-policies.
Same as in the previous experiments, we introduce imbalance by underrepresenting one of the sub-policies to have $9$ demos in \textit{picking} and $10$ demos in \textit{opening}.
To ensure the total number of demos is the same, we increase the number of demonstrations for other sub-policies to $27$ in \textit{picking} and $20$ in \textit{opening}.
We manually upsample the underrepresented demos to uniformly balance the proportion of data.
We train a new policy on this manually balanced dataset for our test.

\p{Results} \fig{manual_balanced} summarizes the results of the experiments. The plots show the success rate for the different sub-policies in the dataset.
The first three plots show the results for \textit{picking}, and the last two plots show the results for \textit{opening}.
In each plot, the gray bar shows the success rate of a policy that was trained on a balanced dataset, and the orange bar represents the success rate of a policy trained on the upsampled dataset, where underrepresented sub-policies were manually upsampled to achieve an equal proportion of data.
Overall, the performance is comparable in both cases.
We did not observe any statistically significant difference in success rates between the balanced and upsampled policies, except in one case: for \textit{picking}, performance decreases when the red block is on the right ($t=8.663, p < 0.001$)
This can be attributed to the relative difficulty of the behaviors.
Picking up the red block from the right is more difficult because of the switch which often obstructs the robot's path.
In this case, successful learning requires more diverse demonstrations, and naive upsampling of the same data can lead to overfitting and subsequently poorer performance.

From these results we conclude that upsampling underrepresented sub-policies can achieve performance comparable to a policy trained on a balanced dataset, but only when all behaviors are relatively of similar difficulty.
Furthermore, this naive approach treats sub-policies as independent, i.e., it ignores the potential overlap in behaviors.
For instance, in the \textit{picking} task demonstrations for picking the block from the left and right locations still provide useful information for learning to pick it from the middle \cite{dai2025civil}.
Ideally, we would want to weight the different demonstrations based on how much they contribute to learning the behaviors they represent.
As discussed in Section \ref{sec:5b}, this can be achieved by comparing how the demonstrations for the different sub-policies affect learning relative to a target loss $\mathcal{L}^{ref}$.

\begin{figure*}
    \centering
    \includegraphics[width=1\linewidth]{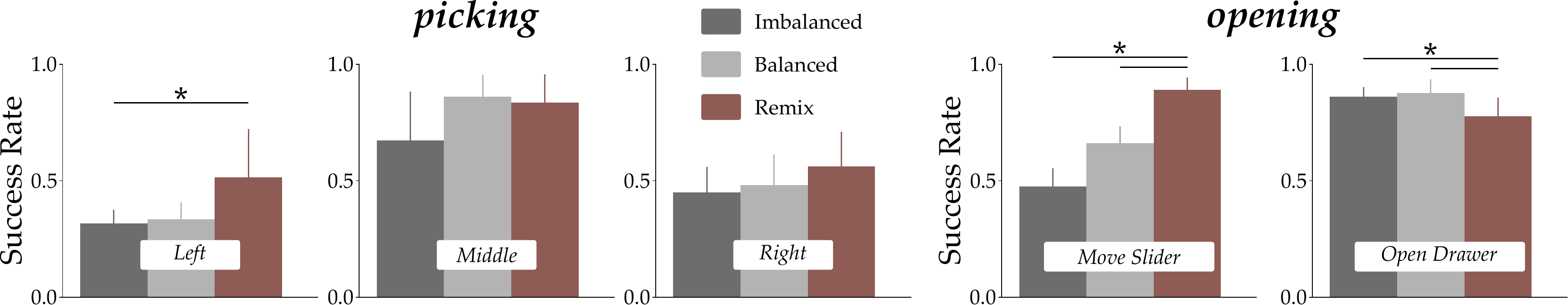}
    \caption{Results of the experiment testing existing approaches for reweighting/balancing datasets. We train three policies. First policy is trained on the imbalanced dataset, In \textit{picking}, the imbalanced dataset contains $27$ demonstrations for the left and right positions of the red block and $9$ demonstrations for the middle position of the block. In \textit{opening} the imbalanced dataset contains $10$ demonstrations for moving the slider and $20$ demonstrations for opening the drawer. We train a second policy on a balanced dataset that contains equal proportions of demonstrations for all sub-policies, i.e., $21$ for all three block positions and $15$ for both \textit{opening} sims. These two policies serve as baselines for comparing the performance of reweighting algorithm Remix \cite{hejna2024re}. Finally, we use Remix to balance the dataset and train a third policy on it. We compare the success rates of the three trained policies across $100$ rollouts. For reliable results we perform the experiments for $10$ trials. The first three plots show the success rates in \textit{picking} and the last two plot show the success rates of the three policies in \textit{opening}. The vertical bars show the standard deviation and $*$ indicates statistical significance.}
    \label{fig:remix}
\end{figure*}

\p{Procedure} In this next experiment, we want to see the effectiveness of the method \textbf{Remix} \cite{hejna2024re} in balancing the dataset.
\textbf{Remix} uses distributionally robust optimization \cite{sagawa2019distributionally} to solve the min-max problem and uses a reference policy as the target loss $\mathcal{L}^{ref}$.
This reference policy is trained on the given imbalanced dataset.
As in the previous experiments, we collect an imbalanced dataset where there are certain behaviors that are underrepresented. Particularly, in \textit{picking} the dataset consisted of $27$, $9$, and $27$ demonstrations for the sub-policies $\pi_1^p, \pi_2^p, \pi_3^p$, respectively; in \textit{opening} the dataset consisted of $10$, $20$ demonstrations for $\pi_1^o$, $\pi_2^o$ respectively.
We follow the procedure discussed by the authors in their paper for reweighting the dataset.
After reweighting the data we train a robot policy on the balanced dataset.
Additionally, we train two baseline policies --- one on the imbalanced dataset and one on a dataset that has equal proportions of different behaviors.
We compare the performance of the policy trained on the remix weighted dataset with these baselines.

\p{Results} We summarize the results of our experiments in \fig{remix}.
The plots compare the success rate of the policies trained on the imbalanced, balanced, and re-weighted datasets. 
We see that after reweighting the dataset using \textbf{Remix} there is an overall improvement in performance across all the sub-policies in both \textit{picking} and \textit{opening} tasks.
Particularly, in \textit{picking} the policy trained on the reweighted dataset is able to better learn each of the three sub-policies (left: $t=-2.789 , p < 0.05$). 
Interestingly, reweighting the data is also able to outperform the policy trained on the balanced dataset for some sub-policies.
In \textit{opening}, since the first sub-policy was underrepresented in the imbalanced dataset, the imbalanced policy is not able to adequately learn that behavior and the second sub-policy dominates.
However, the policy trained on the reweighted dataset can learn the first sub-policy significantly better (comparison with imbalanced: $t=-10.785, p < 0.001$, comparison with balanced: $t=-7.565, p < 0.001$).
The reweighted policy is also able to achieve better performance than the balanced policy.
However, this does come at the cost of a detriment in performance in the second behavior (comparison with imbalanced: $t=3.247$, $p < 0.05$, comparison with balanced: $t=4.023, p < 0.01$).

\begin{figure*}
    \centering
    \includegraphics[width=1\linewidth]{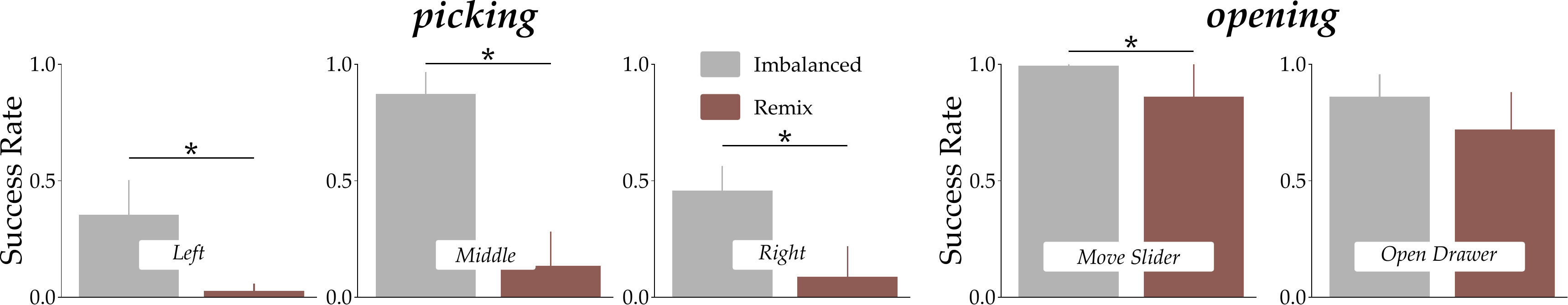}
    \caption{Results of the experiment that highlight limitations of existing balancing approaches. We have an imbalanced dataset that contains optimal and suboptimal demonstrations: in \textit{picking} the dataset contains $100$ optimal and $50$ suboptimal demonstrations. In these demonstrations the block is uniformly sampled from the three locations. In \textit{opening} the dataset contains $20$ optimal and $10$ suboptimal demonstrations for each of the two sub-policies totaling $60$ demonstrations. We use Remix \cite{hejna2024re} to balance the dataset. We train a policy on the imbalanced dataset as a reference and another policy on the balanced dataset. We perform this for $10$ trials. To compare the performance of these two learned policies we measure the success rate across $100$ rollouts. The vertical bars show the standard deviation and $*$ indicates statistical significance.}
    \label{fig:remix_suboptimal}
\end{figure*}

These results demonstrate the merits of methods like remix that can balance the dataset to re-weight demonstrations that are more difficult to learn.
This can enhance the policy's ability to learn every behavior present in the dataset regardless of whether they are under-represented.
Moreover, it can also utilize the dataset more efficiently by learning behaviors that are common across all demonstrations, for example goal-reaching actions.
But this is contingent on the choice of the reference target $\mathcal{L}^{ref}$.
In these experiments, since $\mathcal{L}^{ref}$ is trained on an imbalanced dataset, it is unable to learn the under-represented sub-policy correctly.
This increases the excess loss for the corresponding sub-policy allowing \textbf{Remix} to re-weight the dataset appropriately, i.e., it can help equally represent the corresponding behavior in the final policy.
However, there is a limitation to such an approach --- when the imbalanced dataset contains a small amount of suboptimal data $\mathcal{L}^{ref}$ can produce counter-intuitive results by increasing the excess loss for the suboptimal data.
 
\p{Procedure} To demonstrate this, we conduct another set of experiments in the \textit{picking} and \textit{opening} settings.
We utilize an imbalanced dataset with majority optimal data and a small amount of suboptimal data. 
Specifically, for \textit{picking} we supply $100$ optimal trajectories and $50$ suboptimal demonstrations where the block is randomly initialized in one of the three regions.
In \textit{opening}, we use $20$ optimal and $10$ suboptimal demonstrations for both $\pi_1^o$ and $\pi_2^o$ resulting in a total of $60$ demonstrations.
We repeat the previous procedure with this new mixture of data for the two experiments.

\p{Results} We present the results in \fig{remix_suboptimal}.
The three plots on the left correspond to the success rate of the three behaviors in \textit{picking}, while the two plots on the right correspond to the two behaviors observed in \textit{opening}.
The gray bars show the success rates of the policy trained on the imbalanced data and the brown bars show the success rates of the policy trained on the re-weighted data.
We see that there is a drop in performance for all the behaviors in the two experiments.
We found statistical significance for all cases except the behavior of opening a drawer (left: $t=5.978, p < 0.001$, middle: $t=11.026, p < 0.001$, right: $t=8.164, p < 0.001$, move the slider: $t=2.407, p < 0.05$).

These results present the failure case of \textbf{Remix}; when the dataset contains a small amount of suboptimal data, the reference policy trained on the imbalanced dataset is biased by the optimal data which is relatively abundant.
In other words, the reference policy is able to reliably learn the behaviors.
When using such a reference policy as the target loss $\mathcal{L}^{ref}$, the min-max optimization is compelled to down-weight the optimal data and up-weight the suboptimal data.
Consequently, a policy trained on a dataset balanced by Remix results in poorer performance than the policy trained on the imbalanced dataset.
\section{Learning Desired Balance}

\begin{figure*}
    \centering
    \includegraphics[width=1\linewidth]{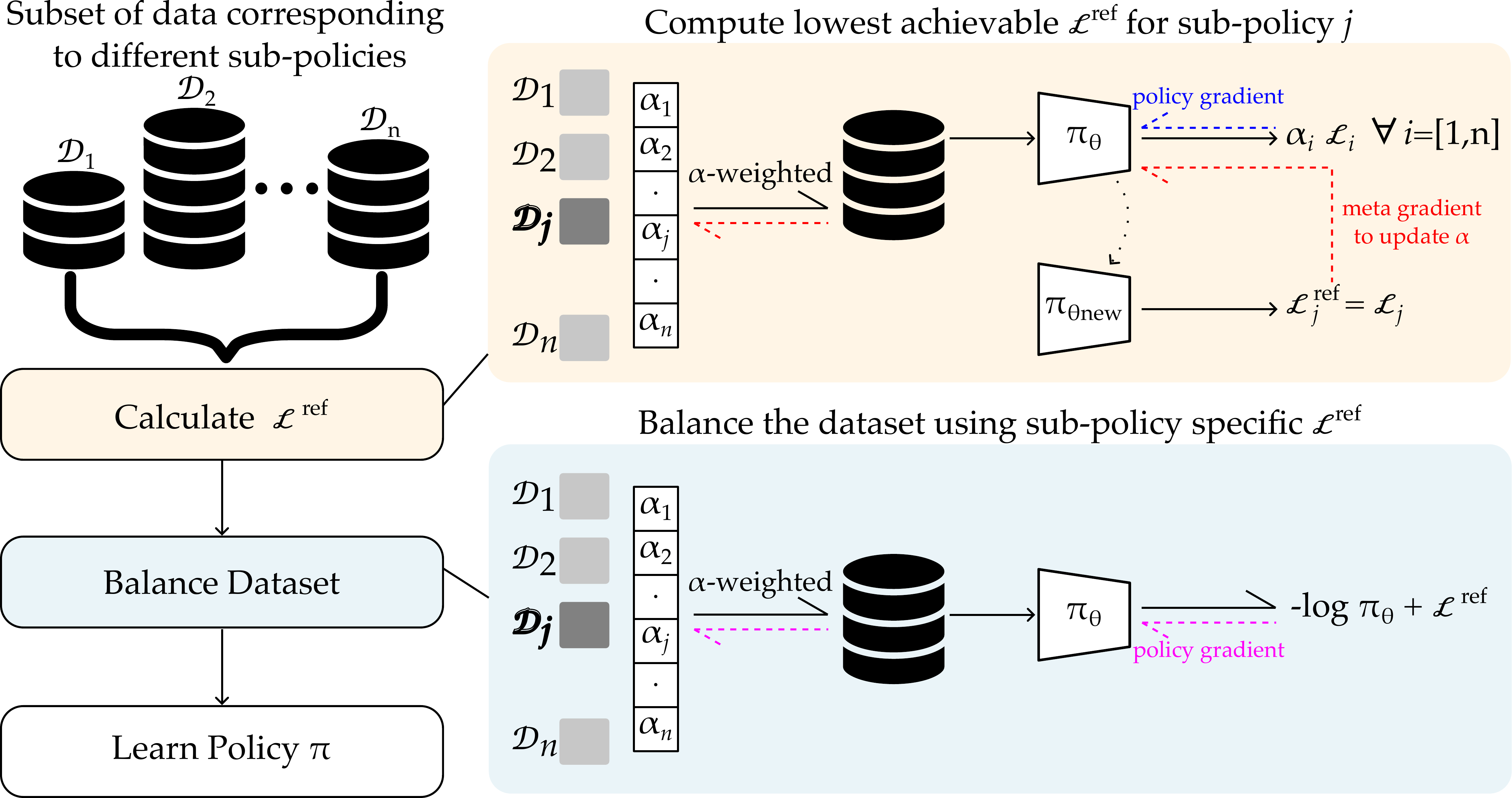}
    \caption{Schematic diagram of our method. The left side shows the workflow of our approach. Given an imbalanced dataset collected using multiple sub-policies or behaviors, we first aim to learn a target loss $\mathcal{L}^{ref}_j$ corresponding to each sub-policy $j \forall j = 1, \cdots, n$. Here $n$ is the number of sub-policies used to collect the imbalanced dataset. First, we must train a policy that achieves the best possible performance on sub-policy $j$. To this end, we optimize the objective given in \eq{P16} to find the best possible $\alpha^*_j$ corresponding to the sub-policy $j$. In our meta-gradient approach, we first update the parameters $\theta$ of the robot policy and use the updated policy to calculate the loss corresponding to the specific sub-policy j. The meta-gradient of this loss is then used to update the weights $\alpha$. The arrows indicate how the gradients pass through the different parameters. This meta-gradient approach is repeated to independently calculate the target loss for each sub-policy in the dataset.} Once we have calculated the target loss for each of the sub-policies, we balance the dataset by optimizing \eq{P13} and using our calculated $\mathcal{L}^{ref}$. Finally, we train a robot policy on the balanced dataset.
    \label{fig:method}
\end{figure*}

Based on related works within and outside of robotics we have discussed two ways of balancing a dataset. 
The first is a straightforward method that assigns equal weight to each behavior in the dataset. While this can be effective, it fails to recognize that some behaviors are inherently easier to learn than others.
We then described a more general approach that dynamically weighs the behaviors based on their training loss to account for differences in their learning difficulty.
This method requires specifying a reference loss that represents the desired level of accuracy for each behavior. 
Although prior research proposes reasonable assumptions for this reference, we experimentally show that these can be counterproductive in certain instances.
To address this limitation, we now discuss how the target accuracy can be determined in a more principled manner.

\subsection{Learning Reference Loss}\label{sec:6a}

As we demonstrated in the experiments in Section~\ref{sec:5c}, overestimating the target can cause the learner to solely focus on the hardest behavior --- either making desirable progress on it or, undesirably, degrading the performance on easier behaviors, until they all converge to the same accuracy. 
On the other hand, setting a conservative target may prevent the robot from improving on that behavior by increasing its weight.

Instead of relying on practical assumptions, we propose estimating the lowest achievable training loss for each behavior and setting it as its reference. More specifically, we focus on learning one sub-policy at a time and determine what would be the lowest training loss we can achieve if we re-weight the entire dataset just to learn this one sub-policy. Repeating this for each of the sub-policies ensures that we get a lower-bound estimate of the reference loss for each sub-policy that accurately accounts for the difficulty of learning the behavior from the given dataset.
We start by defining the expected loss for behavior $\pi_{i}$:
\begin{equation}\label{eq:P16}
    \mathcal{L}_{i} = \mathbb{E}_{(s,a)\sim\mathcal{D}_{i}} D_{KL}(\pi_{i}||\pi_{\theta_{\alpha}})
\end{equation}
Here $\pi_{\theta_{\alpha}}$ represents the policy trained on dataset $\mathcal{D}$ using weights $\alpha$ by optimizing: 
\begin{equation}\label{eq:P17}
    \mathcal{L}_{\alpha}(\theta) =  \sum_{i=1}^{k} \alpha_{i} \left(\mathbb{E}_{(s,a)\sim\mathcal{D}_{i}} D_{KL}(\pi_{i}||\pi_{\theta_{\alpha}})\right)
\end{equation}

To determine the minimum value for \eq{P16}, we must find weights $\alpha^{*}$ that make the best use of the entire dataset for learning $\pi_{i}$.
Note that $\alpha^{*}$ may not necessarily assign full weight to the data $\mathcal{D}_{i}$ for that behavior, i.e., the weights need not be $\alpha^{*}_{i} = 1$ and $\alpha^{*}_{i\neq i} = 0$. In many cases, incorporating information from other behaviors can actually improve the robot's ability to learn the target behavior.
To learn these weights and compute the reference loss for each behavior, we employ the following meta-gradient approach:

We first update the policy parameters using the loss in \eq{P17} for one (or a few) gradient step(s):
\begin{equation}\label{eq:P18}
    \theta_{new} = \theta - \beta_{1} \nabla_{\theta}\mathcal{L}_{\alpha}(\theta)
\end{equation}
Next, we update the weights by computing the gradients based on \eq{P16} and passing them back through the previous steps in \eq{P18}:
\begin{equation}\label{eq:P19}
    \alpha_{new} = \alpha - \beta_{2} \nabla_{\alpha}\mathcal{L}_{i}(\theta_{new})
\end{equation}

\begin{figure*}
    \centering
    \includegraphics[width=1\linewidth]{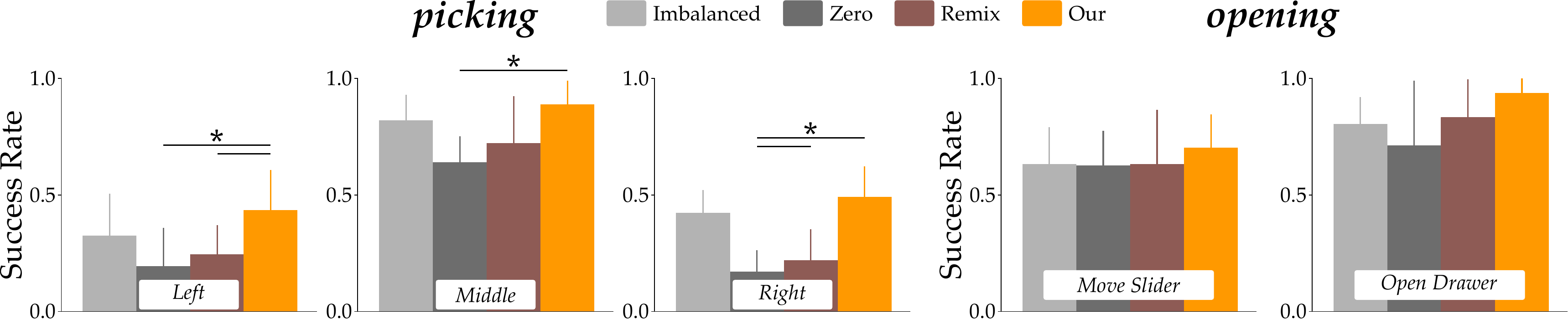}
    \caption{Testing our proposed approach against the baselines. We train four policies on a mixed dataset containing optimal and suboptimal demonstrations. The \textbf{Imbalanced} policy is trained on the original dataset without any modifications. In \textit{picking}, the imbalanced dataset contains $30$ suboptimal demonstrations and $60$ optimal demonstrations. In these demonstrations the initial position of the red block is uniformly sampled from the three regions. In \textit{opening} the imbalanced dataset contains $10$ suboptimal and $20$ optimal demonstrations for moving the slider and opening the drawer. We use different target loss $\mathcal{L}^{ref}$ in the objective given in \eq{P13} to balance the dataset and learn a policy. For \textbf{Zero} we set the target loss to be zero, i.e., $\mathcal{L}^{ref} = 0$. For \textbf{Remix} we follow the procedure given in \cite{hejna2024re} and use a reference policy for calculating the target loss. This reference policy is trained on the given imbalanced dataset. Finally, we train a policy on the dataset balanced using \textbf{Our} proposed meta-gradient approach. We compare the performance of the four policies using success rate as the metric. Specifically, perform $10$ trials and execute each of the trained policy for a $100$ rollouts to calculate the success rates. The plots show the success rate of the four policies --- imbalanced, Zero, Remix, and Our --- in \textit{picking} (first three plots) and \textit{opening} (last two plots). The vertical bars show the standard deviation and $*$ indicates statistical significance.}
    \label{fig:results_method}
\end{figure*}

Our proposed approach is summarized in \fig{method}.
We use the converged weights $\alpha_{new} = \alpha^{*}$ to compute $\mathcal{L}_{i}^{min}$ and set it as the reference $\mathcal{L}_{i}^{ref}$ for that behavior. 
The reference losses $\mathcal{L}_{1}^{min}, \ldots, \mathcal{L}_{k}^{min}$ computed by our approach represent the best training performance for each behavior, avoiding the pitfalls of overestimating or underestimating the targets.
Plugging these target losses back into the objective in \eq{P13} yields a \textit{balance} where the robot's policy imitates each behavior with an equal margin of error from its best-possible loss.
The following experiments demonstrate how our approach improves upon previous methods for balancing datasets based on training loss.

\subsection{Comparing with Baselines}\label{sec:6b}

In this section we test our proposed method for calculating the target $\mathcal{L}^{ref}$ for different sub-policies present in the dataset.
We intend to examine how the choice of $\mathcal{L}^{ref}$ affects the data balancing scheme and the final policy's performance.

\p{Procedure} We use the two experiments \textit{picking} and \textit{opening} from the previous sections.
To train a policy, we collect a dataset with a small amount of suboptimal data.
Particularly, in the case of \textit{picking} the dataset contains $60$ optimal demonstrations and $30$ suboptimal demonstrations, and for \textit{opening} the dataset contains $20$ optimal demonstrations and $10$ suboptimal demonstrations for the two behaviors --- opening a drawer and moving a slider.
We first train a policy on the imbalanced dataset, then use different approaches to balance the data and subsequently train a policy on this balanced dataset.

\p{Independent Variables} We compare three different ways of calculating the target loss. First, we use \textbf{Our} method which uses a meta-gradient approach to calculate the target loss $\mathcal{L}^{ref}$.
We compare this with \textbf{Remix} and \textbf{Zero}.
As mentioned previously, \textbf{Remix} uses a reference policy that is trained on the imbalanced dataset as the target loss.
On the other hand, \textbf{Zero} does not use a target loss for the sub-policies, i.e., $\mathcal{L}^{ref} = 0$.

\begin{table*}[]
\centering
\begin{tabular}{clllll}
\hline
\multicolumn{1}{l}{} & \multicolumn{1}{c}{Method} & \multicolumn{2}{c}{Suboptimal Data} & \multicolumn{2}{c}{Optimal Data} \\
\multicolumn{1}{l}{} & \multicolumn{1}{c}{} & \multicolumn{1}{c}{$\alpha$} & \multicolumn{1}{c}{$\mathcal{L}^{ref}$} & \multicolumn{1}{c}{$\alpha$} & \multicolumn{1}{c}{$\mathcal{L}^{ref}$} \\ \hline
\multirow{4}{*}{\textit{picking}} & \textbf{Imbalanced} & 0.33 & - & 0.67 & - \\
 & \textbf{Zero} & 0.5 & 0 & 0.5 & 0 \\
 & \textbf{Remix} & 0.5348 & 0.520 & 0.4652 & 0.9337 \\
 & \textbf{Our} & \textbf{0.0633} & 1.893 & \textbf{0.9367} & 1.505 \\ \hline
\multirow{4}{*}{\textit{opening}} & \textbf{Imbalanced} & 0.333 & - & 0.667 & - \\
 & \textbf{Zero} & 0.516 & 0 & 0.484 & 0 \\
 & \textbf{Remix} & 0.5 & 0.089 & 0.5 & 0.0574 \\
 & \textbf{Our} & \textbf{0.0563} & 2.018 & \textbf{0.9437} & 1.626 \\
\hline
\end{tabular}
\caption{The table shows the assigned weights learned by \textbf{Our} method and the baselines. Additionally, it shows the calculated target loss $\mathcal{L}^{ref}$ corresponding to the optimal and suboptimal demonstrations. This target loss is used to determine the relative accuracy we want to achieve corresponding to the different behaviors in the dataset. A lower value of $\mathcal{L}^{ref}$ means we want to achieve a higher accuracy on the respective behavior.}
\label{tab:weights_6}
\end{table*}

\p{Results} We summarize the results of our experiments int \fig{results_method}.
The plots illustrate the success rate achieved by the policy for the different sub-policies present in the dataset in the \textit{picking} experiments (three plots on the left) and in the \textit{opening} experiment (two plots on the right).
We see that \textbf{Our} method is able to achieve better performance than the baselines:
while \textbf{Our} method improves the performance of the policy compared to the imbalanced policy, the baselines see a reduction in the performance.
A repeated measures ANOVA revealed that the methods had a significant effect on the success rate in \textit{picking} ($F(3, 27) = 11.045, p < 0.001$), but not in \textit{opening} ($F(3, 27)=1.739, p = 0.183$).
In \textit{picking}, \textbf{Our} saw a significantly better performance for all three sub-policies than \textbf{Zero} (left: $p < 0.05$, middle: $p < 0.05$, right: $p < 0.001$), whereas compared to \textbf{Remix} our method only saw significance for left ($p < 0.001$) and right locations ($p < 0.001$).
In \textit{opening}, while we did not find statistical significance, we do see that \textbf{Our} method is able to achieve higher performance than the baselines.

Additionally, we report the weights and the target losses for the optimal and suboptimal datasets in Table \ref{tab:weights_6}.
We see that the baselines \textbf{Zero} and \textbf{Remix} assign equal weights to the optimal and suboptimal data, which leads to poor performance when trained on the reweighted data.
In contrast, \textbf{Our} method upweights the optimal data, resulting in a policy that achieves higher performance than the baselines.
From the weights it is evident that \textbf{Our} method relies on nearly all of the optimal data while using very little of the suboptimal data.
This explains the improvement in performance of our policy compared to the imbalanced policy.
The lack of significant improvement in \textit{opening} could be attributed to the simplicity of the tasks: since the desired object (slider and drawer) are always at the same location, the negative impact of a small amount of suboptimal data is not that prominent. 
Consequently, there is only a slight improvement after balancing.

\subsection{Limitations of Offline Data Balancing}\label{sec:6c}

Our approach advances the capabilities of current data balancing approaches by automatically determining the desired accuracy for each dataset behavior.
However, the results in Section~\ref{sec:6b} indicate that there is still some gap between the performance achieved using manually selected weights and the learned weights. 

This lack of performance highlights a key limitation of offline data curation approaches. Without access to task-level evaluations, existing approaches, including our proposed method, are limited to evaluating the accuracy of imitating the data samples.
However, imitation accuracy is not always proportional to task success.
Returning to the example of dropping a ball into a large bin versus throwing it at a moving target: even if the actions for the dropping task are learned less accurately during training, we may be able to perform the task successfully at test time. In contrast, even minimal imitation errors may lead to failures in the throwing task.

Ideally, instead of equalizing training loss across behaviors, we would balance their \textit{test} performance (i.e., the performance of the policy during rollouts).
Achieving this offline would require either real-world deployment, a high-fidelity simulator, or a reliable task-level value function --- resources that are typically expensive or impractical.
We therefore advocate for future efforts that focus on efficiently utilizing real-world interactions.
One promising direction is to start with a conservative reference and iteratively relax it based on human feedback.
If the reference remains above $\mathcal{L}_{i}^{min}$, we set the target to be $\mathcal{L}_{i}^{ref} = \mathcal{L}_{i}^{success}$. Otherwise, we retain the original target, i.e., $\mathcal{L}_{i}^{ref} = \mathcal{L}_{i}^{min}$ to prevent the robot from overestimating its learning capacity.
We look forward to future efforts that address this gap by developing methods that support task-level refinement with minimal supervision.

\section{Conclusion}

Within this manuscript we argue that imitation learning is affected by the distribution of the training data.
Specifically, we explore settings where the robot is learning complex, multi-part behaviors from human demonstrations.
These multi-part behaviors could be based on constraints: e.g., teaching the robot to reach a goal while avoiding obstacle regions.
Alternatively, these behaviors could be composed of multiple tasks: e.g., teaching the robot to manipulate a cup at different locations on a table.
In either case, real-world human demonstrations are inevitably \textit{imbalanced}.
The human provides more state-action pairs for some parts of the task than others.
Returning to our example, perhaps the human provides twice as many state-action pairs that move towards the goal than state-action pairs which move away from the constraint regions.

Existing imitation learning paradigms largely ignore this dataset imbalance.
Within behavior cloning --- for instance --- every state-action pair is given equal weight by the loss function.
We theoretically prove that this default approach results in imbalanced policies, where the robot overly focuses on the most represented behaviors, and may fail to learn behaviors the human shows less frequently.
Our experimental results support this conclusion, and highlight how --- even in simple settings --- dataset imbalance can cause state-of-the-art learning methods to fall short.
We next analyze algorithms that autonomously rebalance the dataset without human intervention.
These algorithms adjust the ratio (or weight) of elements of an offline dataset, and do not require any new data gathering.
Our theoretical and empirical results show that autonomously reweighting the dataset has the potential to improve learned policy performance without changing the robot's learning algorithm.
But we also find that the best method to reweight the dataset depends on several factors.
Our work formulates the pros and cons of different autonomous approaches for reweighting, providing guidelines for future designers.
We conclude by introducing a novel meta-gradient approach for autonomously rebalancing offline datasets.
This approach addresses the primary limitations of current methods, and our experiments highlight that the meta-gradient method improves downstream robot learning.

\p{Future Work}
We see this work as a first step towards understanding the dataset characteristics necessary for effective imitation learning.
We argue that dataset imbalance is a fundamental and practical issue, and this issue is not fully explored by prior works.
In our future research, we are looking into extending this analysis to large-scale datasets and applying our approach to automatically balance these datasets for more effective generalist policies.
\section{Declarations}

\p{Funding} This research was supported in part by the USDA National Institute of Food and Agriculture, Grant 2022-67021-37868.

\p{Conflict of Interest} The authors declare that they have no conflicts of interest.

\p{Author Contribution} S.P.: Conceptualization, Investigation, Software, Methodology, Formal analysis, Writing - original draft.
H.N.: Conceptualization, Investigation, Methodology, Writing - original draft.
D.L.: Conceptualization, Supervision, Funding Acquisition, Writing - review and editing.


\bibliography{bibtex}
\bibliographystyle{spmpsci}

\appendix
\section{Appendix}
\subsection{Proof}\label{sec:appendix}

In Section \ref{sec:5b}, we analyze how to balance a dataset by taking into consideration the relative difficulty of each behavior.
As mentioned previously, we perform a min-max operation with the objective in \eq{P13} where we iterate between maximizing the objective to update the weights $\alpha$ and minimizing it to learn the policy parameters $\theta$.
Here we detail how we compute the gradient of the objective with respect to $\alpha$ (as shown in \eq{P14}) and derive the convergence condition for the maximization step stated in \eq{P15}.

\begin{figure*}
    \centering
    \includegraphics[width=1\linewidth]{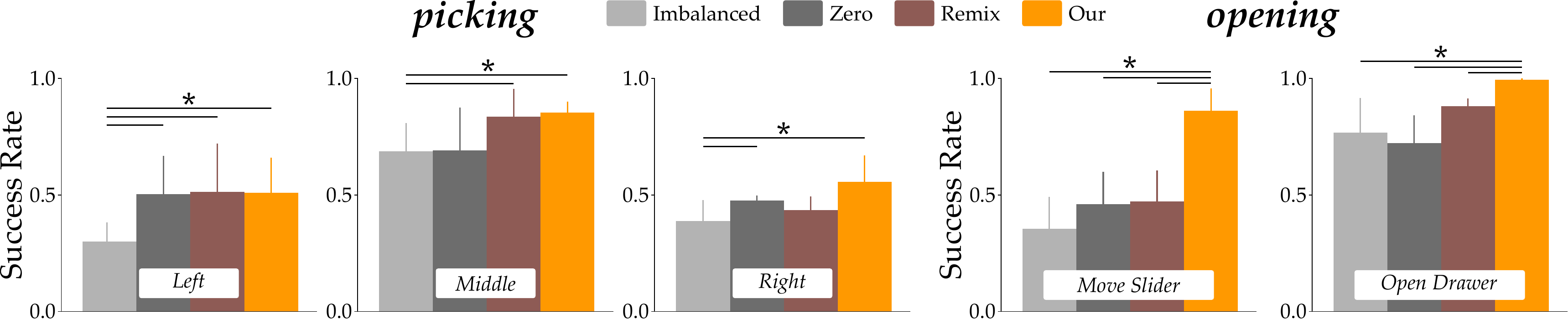}
    \caption{Testing our proposed approach against the baselines. We train four policies on a mixed dataset containing optimal demonstrations. The \textbf{Imbalanced} policy is trained on the original dataset without any modifications. In \textit{picking}, the imbalanced dataset contains $27$ demonstrations for the left and right locations of the red block but only $9$ demonstrations for the middle location. In these demonstrations the initial position of the red block is uniformly sampled from the three regions. In \textit{opening} the imbalanced dataset contains $10$ suboptimal and $20$ optimal demonstrations for moving the slider and opening the drawer. We use different target loss $\mathcal{L}^{ref}$ in the objective given in \eq{P13} to balance the dataset and learn a policy. For \textbf{Zero} we set the target loss to be zero, i.e., $\mathcal{L}^{ref} = 0$. For \textbf{Remix} we follow the procedure given in \cite{hejna2024re} and use a reference policy for calculating the target loss. This reference policy is trained on the given imbalanced dataset. Finally, we train a policy on the dataset balanced using \textbf{Our} proposed meta-gradient approach. We compare the performance of the four policies using success rate as the metric. Specifically, perform $10$ trials and execute each of the trained policy for a $100$ rollouts to calculate the success rates. The plots show the success rate of the four policies --- imbalanced, Zero, Remix, and Our --- in \textit{picking} (first three plots) and \textit{opening} (last two plots). The vertical bars show the standard deviation and $*$ indicates statistical significance.}
    \label{fig:appendix}
\end{figure*}

We recognize that $\alpha$ represents the \textit{relative weights}, which makes each maximization step the following constrained optimization problem:
\begin{align*}
    &\max_{\alpha} \sum_{i=1}^{k} \alpha_{i} \left( \mathbb{E}_{(s,a) \sim\mathcal{D}_{i}} D_{KL} (\pi_{i} || \pi_{\theta}) - \mathcal{L}_{i}^{ref} \right)\\
    &\text{s.t.} \quad \sum_{i=1}^{k} \alpha = 1
\end{align*}

We solve the above problem using the method of Lagrange multipliers, where the Lagrangian is defined as:
\begin{align}\label{eq:P20}
    \mathcal{L}(\alpha, \lambda) =& \sum_{i=1}^{k} \alpha_{i} \left( \mathbb{E}_{(s,a) \sim\mathcal{D}_{i}} D_{KL} (\pi_{i} || \pi_{\theta}) - \mathcal{L}_{i}^{ref} \right) \nonumber\\
    &+ \lambda \left(\sum_{i=1}^{k} \alpha - 1\right)
\end{align}
Here $\lambda$ is the Lagrange multiplier. We then take the partial derivative of \eq{P20} with respect to the weight $\alpha_{i}$ of the $i$-th behavior to get:
\begin{align}\label{eq:P21}
    \frac{\partial \mathcal{L}}{\partial \alpha_{i}} &= \mathbb{E}_{(s,a) \sim\mathcal{D}_{i}} D_{KL} (\pi_{i} || \pi_{\theta}) - \mathcal{L}_{i}^{ref} + \lambda \nonumber\\
    &= \delta_{i} + \lambda
\end{align}

Doing so for each behavior results in \eq{P14} reported in Section~\ref{sec:5b}. Next, to solve for $\lambda$, we apply the constraint that the new weights should also sum to $1$:
\begin{equation*}
    \sum_{i=1}^{k} \alpha_{new, i} =     \sum_{i=1}^{k} \alpha_{i} - \eta \frac{\partial \mathcal{L}}{\partial \alpha_{i}} = 1
\end{equation*}
where $\eta$ is the learning rate. Since the previous $\alpha_i$ also sum to $1$, we get that:
\begin{align}
    \eta \sum_{i=1}^{k} \frac{\partial \mathcal{L}}{\partial \alpha_{i}} &= 0 \nonumber \\
    \sum_{i=1}^{k} \delta_{i} + \lambda &= 0 \nonumber\\
    \lambda &= - \frac{1}{k} \sum_{i=1}^{k} \delta_{i}
\end{align}

Finally to derive the convergence condition in \eq{P15}, we equate the partial derivatives in \eq{P21} to $0$:
\begin{align}\label{eq:P23}
    \frac{\partial \mathcal{L}}{\partial \alpha_{i}} = \delta_{i} + \lambda &= 0 \nonumber \\
    \delta_{i} &= -\lambda
\end{align}
In this way, we obtain the result that the maximization step converges when all differences are equal, i.e., $\delta_{1} = \ldots = \delta_{k} = \frac{1}{k} \sum_{i=1}^{k} \delta_{i}$.

\subsection{Simulation}\label{sec:appendix_2}
In Section \ref{sec:6c} we perform experiments to showcase the limitations of existing data balancing approaches.
Specifically, we test different metrics for the target loss $\mathcal{L}^{ref}$ and show that our meta-gradient approach outperforms the baselines. 
In this section, we test our method against baselines \textbf{Remix} and \textbf{Zero} in a setting where the dataset is optimal.

\p{Procedure} We repeat the procedure from Section \ref{sec:6c} where we use two experiments \textit{picking} and \textit{opening}.
We have a dataset with optimal demonstrations: for \textit{picking} there are $27$ demonstrations for the left and right positions of the block and $9$ demonstrations for the middle position; for \textit{opening} there are $10$ demonstrations for moving the slider and $20$ demonstrations for opening the drawer.
We train a policy on the imbalanced dataset and use \textbf{Zero}, \textbf{Remix}, and \textbf{Our} methods to balance the dataset.
We then train a policy on the balanced dataset and compare the performance of each learned policy.

\p{Results} The plots in \fig{appendix} summarize the results. We compare the success rate achieved by the learned policy for different behaviors in the two settings of \textit{picking} and \textit{opening}. 
A repeated measures ANOVA revealed that the methods had significant effect on the success rate (\textit{picking}: $F(3, 27)=7.790, p<0.001$, \textit{opening}: $F(3, 27) = 49.407, p < 0.001$).
We find that while all three methods are able to achieve better performance than the policy trained on the imbalanced dataset, only \textbf{Our} method achieves statistical significance for all the behaviors in both experiments.
Particularly, \textbf{Zero} is significantly better than the imbalanced policy for the left ($p < 0.05$) and right ($p < 0.05$) locations of the red block but performs as poorly as the imbalanced policy when the block is in the middle of the table.
Similarly, \textbf{Remix} does significantly better than the imbalanced policy for left ($p < 0.05$) and middle ($p < 0.05$) locations of the red block but is only slightly better for the locations on the right.
In contrast, \textbf{Our} method achieves significance for all three behaviors ($p < 0.05$).
In \textit{opening}, we see that while \textbf{Remix} achieves better performance than the imbalanced policy for both behaviors, \textbf{Our} method is significantly better than all three baselines ($p < 0.05$).

Overall, these results suggest that not only is our method able to balance heterogeneous optimal data to achieve better performance, our previous experiments from Section \ref{sec:6c} indicates that it can also overcome the limitations of state-of-the-art methods like Remix.
\end{document}